\documentclass{article}
\usepackage{spconf}
\ninept

\usepackage{graphics}
\usepackage{amsthm}
\usepackage{amsmath}
\usepackage{amssymb}
\usepackage{color}
\usepackage{parskip}

\newtheorem{definition}{Definition}
\newtheorem{proposition}{Proposition}
\usepackage{tikz}
\usepackage{pgfplotstable}
\pgfplotsset{compat=1.14}
\usepgfplotslibrary{statistics}
\usepackage{makecell}
\usepackage{hyperref}
\usepackage[utf8]{inputenc}

\title{Structural Robustness for Deep Learning Architectures}

\name{Carlos Lassance$^{\star,\dagger}$ \qquad Vincent Gripon$^{\star,\dagger}$\qquad Jian Tang$^{\diamond}$\qquad Antonio Ortega$^{\circ}$}

\address{$^{\star}$ Université de Montréal, Mila\\ $^{\dagger}$ IMT Atlantique, Lab-STICC \\ $^\diamond$ HEC Montréal, Mila \\ $^{\circ}$ University of Southern California\\}

\newcommand{\vx}{\mathbf{x}}
\newcommand{\vy}{\mathbf{y}}
\newcommand{\vc}{\mathbf{c}}
\newcommand{\vepsilon}{\mathbf{\varepsilon}}

\begin{document}

\maketitle

\begin{abstract}
Deep Networks have been shown to provide state-of-the-art performance in many machine learning challenges. Unfortunately, they are susceptible to various types of noise, including adversarial attacks and corrupted inputs. In this work we introduce a formal definition of robustness which can be viewed as a localized Lipschitz constant of the network function, quantified in the domain of the data to be classified. 
We compare this notion of robustness to existing ones, and study its connections with methods in the literature. We evaluate this metric by performing experiments on various competitive vision datasets.
\end{abstract}

\section{Introduction}

In recent years it has been shown that Deep Learning Architectures can provide state-of-the-art performance in many machine learning challenges, ranging from domains as distinct as vision~\cite{alexnet,resnet,shakeshake} and natural language processing~\cite{nmt,elmo}. This success can be justified based on their universal approximation
properties~\cite{universalapproximators}, which allow them to approximate any function that associates each training set input to its corresponding class. But this is also a double-edged sword, as the resulting function may not handle well domain shifts (i.e., it does not generalize well to previously unseen inputs). Adversarial attacks (i.e., imperceptible changes to the input built specifically to fool the network function)~\cite{szegeny,goodfellowadversarial} illustrate the risks of bad generalization. Isotropic noise~\cite{mallat} or corrupted inputs~\cite{hendrycks2019robustness} are also likely to produce similar misclassifications. In applications that are very sensitive to errors, such as autonomous vehicles or robotic assisted surgery, robustness to such deviations is a key challenge.

In the literature, several methods have been proposed to increase the robustness of network functions. A first set of approaches proposes to artificially increase the size of the training set by augmenting it with corrupted inputs~\cite{ford2019adversarial,madry,ladder,kurakin}. Then, during the training phase the network function becomes increasingly robust to the corresponding corruptions.
However, there is no guarantee that increasing robustness to a specific type of corruptions leads to better performance on other types of corruptions, as discussed in~\cite{hendrycks2019robustness,engstrom2018evaluating}.

To achieve universal robustness, other approaches target structural properties of the network function, such as constraining its Lipschitz constant to be small. Recall that a function $F$ is said to be $\alpha$-Lipschitz with respect to a norm $\|\cdot\|$ if $\|F(y)-F(x)\| \leq \alpha \|y-x\|, \forall x,y$. Provided $\alpha$ is small, such a function is robust to small deviations around correctly classified inputs, as it holds that: $\|F(x+\varepsilon)-F(x)\| \leq \alpha\|\varepsilon\|$. For example in Parseval Networks~\cite{parseval}, the authors softly enforce the network $\mathcal{L}_2$ and $\mathcal{L}_\infty$ Lipschitz constants to be bounded. 
Another example is~\cite{l2nonexpansive} where the authors propose to bound only the $\mathcal{L}_2$ norm of the network. 

Yet imposing a small Lipschitz constraint on a network function can be problematic. Indeed, the Lipschitz constraint is controlling the slope of the function {\em everywhere in the input space}. However, in the context of classification we expect that there can be sharp transitions in the network function near class boundaries, while the network function should be smooth away from the boundaries. In other words, since the smoothness properties of the network function are location-dependent (e.g., different behavior close to class boundaries) global Lipschitz metrics may not be meaningful.   
To illustrate this point, consider Figure~\ref{fig:incompatibility_lipschitz}, where we depict the proportion of pairs of training set inputs of distinct classes that are incompatible with a given Lipschitz constraint on the network function, for various datasets~\cite{cifar10,chrabaszcz2017downsampled} and for the $\mathcal{L}_{\infty}$ norm. In this example, the network function is taken to be the one-hot-bit encoded vectors of the corresponding classes. This example illustrates that for such a sharp network function a global Lipschitz constraint is not meaningful: unless the Lipschitz constant is large (e.g., greater than 4) imposing a constraint will prevent the training error from converging to zero.  This example also suggests two related principles that can lead to better robustness and motivate our proposed robustness metric: i) robust network functions should not have sharp transitions in boundary regions, ii) smoothness metrics should be localized. 

\begin{figure}
    \centering
    \begin{tikzpicture}
     \begin{scope}[scale=1]
       \begin{axis}[width=8cm,height=3.7cm,mark size=0.2pt,
           xlabel=Lipschitz constant, ymode=log,
           ylabel=Fraction of pairs, xmin=1]
         \addplot table {lipschitz_estimation/cifar_10/fig1.txt};
         \addlegendentry {CIFAR-10}
         \addplot table {lipschitz_estimation/cifar_100/fig1.txt};
         \addlegendentry {CIFAR-100}
         \addplot table {lipschitz_estimation/imagenet32/fig1.txt};
         \addlegendentry {Imagenet32}
       \end{axis}
     \end{scope}
    \end{tikzpicture}    
    \caption{Depiction of the proportion of pairs of training examples of distinct classes incompatible with a given Lipschitz constraint on the network function, for various datasets and the $\mathcal{L}_{\infty}$ norm.}
    \vspace{-.3cm}
    \label{fig:incompatibility_lipschitz}
\end{figure}
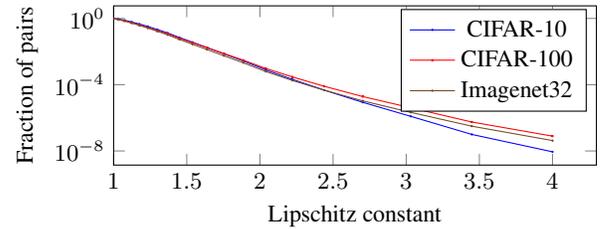

In this paper we introduce a new formal definition of robustness, which can be viewed as a {\em local} Lipschitz constant of the network function in the domain of the training samples. This definition of robustness ensures that any small deviation around a correctly classified input should not dramatically impact the decision of the network function. 
Our proposed definition can be seen as a refinement of previously proposed Lipschitz constraints, where we consider only a small radius around the training inputs, rather than constraining smoothness everywhere in the input space. We then derive reasonable sufficient conditions to enforce robustness of a deep learning architecture and show how these conditions support the use of recently introduced methods~\cite{madry,parseval,l2nonexpansive,ours}. Using experiments on several well-known vision datasets~\cite{cifar10,chrabaszcz2017downsampled} we demonstrate that our proposed definition of robustness is correlated to the robustness observed in a series of existing network training methods.


\section{Robustness definition}
\label{robustness}

Let us consider a (network) function $F$, which maps data (in an input space $\Omega$) into a soft decision for classification.
Thus $F$ is a function from an input vector space (or tensor space) to $\mathbb{R}^C$, where $C$ is typically the number of classes. We denote by $\|\cdot\|$ a fixed metric in the remaining of this work (in most cases either $\mathcal{L}_2$ or $\mathcal{L}_{\infty}$). 

We are interested in the robustness of the network function~$F$. Contrary to Lipschitz constraints, we introduce here a notion of robustness that accounts for:
\begin{enumerate}
    \item The domain $R$ on which it should be defined,
    \item The locality $r$ around each point in $R$ on which it should be enforced.
\end{enumerate}

More formally we define robust behavior as follows: 

\begin{definition}
\label{def:robustness}
We say a network function $F$ is \textbf{$\alpha$-robust} over a domain $R$ and for $r>0$, and denote $F\in {\rm Robust}_\alpha(R,r)$, if:
\begin{equation}
     \|F(\mathbf{x}+\vepsilon) - F(\mathbf{x})\| \leq \alpha \|\vepsilon\|,\forall \mathbf{x} \in R, \forall \vepsilon \text{ s.t. } \|\vepsilon\| < r\;.
\label{robustness_definition}
\end{equation}
\end{definition}

In words, $F\in \rm{Robust}_\alpha(R,r)$ if $F$ is locally $\alpha$-Lipschitz within a radius $r$ of any point in domain $R$. As such, this is equivalent to saying: $F \in \rm{Robust}_\alpha(\Omega,+\infty)$ and $F$ is $\alpha$-Lipschitz. In the remaining of this work, we are interested in enforcing robustness for a small radius $r$ around the training examples $T$.


We also define: $\alpha_{\lim}(F,r) = \inf\{\alpha: F\in \rm{Robust}_\alpha(r)\}$, where $\alpha_{\lim}(F,r)$ represents the minimum value $\alpha$ for which a region of radius $r$ is robust. This allows us to express 
robustness as a trade-off between smoothness slope, as captured by $\alpha$, and radius $r$.

Consider for instance the sigmoid function $\sigma: x\mapsto \frac{1}{1 + \exp(-x)}$ and $R = \{-10,10\}$. Figure~\ref{examplerobustness} (Left) depicts the evolution of $\alpha_{\lim}(\sigma , r)$ as function of $r$. 
We observe that the sigmoid function yields an almost 0-Lipschitz constant around the two points $-10$ and $10$ for a very small radius $r$. When the radius increases, the best Lipschitz constant also increases. The fact that  $\alpha$  is almost 0 when $r$ is small is an illustration of robustness around $R$. The sharp transition occurring for $r\approx 10$ corresponds to the boundary between classes.

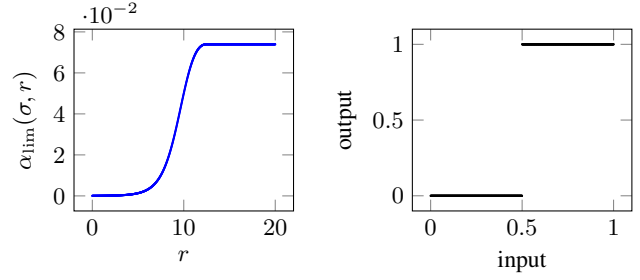
\begin{figure}[ht]
    \centering
   \begin{tikzpicture}
     \begin{scope}[scale=1]
       \begin{axis}[width=4.5cm,height=4cm,mark size=0.2pt,
           xlabel=$r$,
           ylabel={$\alpha_{\lim}(\sigma,r)$}
           ]
         \addplot table {r_alpha_lim.txt};
       \end{axis}
     \end{scope}
     \begin{scope}[scale=1,xshift=4.5cm]
       \begin{axis}[width=4.5cm,height=4cm,mark size=0.2pt,
           xlabel=input,
           ylabel=output
           ]
         \addplot[only marks] table {robustvslip.txt};
       \end{axis}
     \end{scope}

    \end{tikzpicture}    
    \caption{\textbf{Left}: Evolution of $r \mapsto \alpha_{\lim}(\sigma, r)$. \textbf{Right}: Representation of the decision of a mediator hyperplane separator between points $0$ and $1$.}
    \label{examplerobustness}
\end{figure}

\subsection{Relation with Lipschitz constants}

Note that it is immediate, by particularization, that if $F$ is $\alpha$-Lipschitz then $F\in \rm{Robust}_\alpha(r),\forall r$. But if $F\in \rm{Robust}_\alpha(r)$ for some $r$ this does not imply that $F$ is $\alpha$-Lipschitz: to illustrate this, consider a trivial classification problem where the training set is composed of two distinct vectors $\mathbf{x}$ and $\mathbf{x'}$ of distinct classes. A network function $F$ that uses the mediator hyperplane to separate the space into two halves has no Lipschitz constant because of the ``infinite'' slope close to the hyperplane, despite $\alpha_{\lim}(F, \|\mathbf{x}-\mathbf{y}\|_2/2) = 0$. See Figure~\ref{examplerobustness} (Right) for a 1D example.

This is a fundamental result, because the best Lipschitz constant $\alpha$ of a function $F$ is constrained by the dataset, i.e., if two training points of different classes are very close to each other then a zero training error classifier will have a large Lipschitz constant near those points.

%

The proposed robustness criterion is also constrained by the dataset, but allows us to reach any small $\alpha$ for a small enough $r$. Indeed, denote by~$\vc^\vx$ the class corresponding to training example $\vx$. Then, if $F$ matches a nearest neighbor classifier, we obtain that 
\begin{equation}F\in \rm{Robust}_0\left(\min_{\vx,\vx'\in T\atop \vc^\vx\neq\vc^{\vx'}}{\|\vx - \vx'\|/2}\right)\;,\label{0alpha}\end{equation}
and thus any small value for $\alpha$ is achievable within a small radius around examples.
\begin{figure*}[ht]
    \centering
   \begin{tikzpicture}
     \begin{scope}[scale=1]
       \begin{axis}[width=6cm,height=4.5cm,
           xlabel=$\alpha$,
           ylabel=Fraction of pairs, ymax=0, ymin=0.000000001,
           ymode=log,
            title=CIFAR-10,
           ]
         \addplot table {lipschitz_estimation/cifar_10/r_0.9.txt};
         \addlegendentry {$d=0.9$}
         \addplot table {lipschitz_estimation/cifar_10/r_0.7.txt};
         \addlegendentry {$d=0.7$}
         \addplot table {lipschitz_estimation/cifar_10/r_0.5.txt};
         \addlegendentry {$d=0.5$}
         \addplot table {lipschitz_estimation/cifar_10/r_0.3.txt};
         \addlegendentry {$d=0.3$}
       \end{axis}
     \end{scope}
     \begin{scope}[scale=1,xshift=5.5cm]
       \begin{axis}[width=6cm,height=4.5cm,
           xlabel=$\alpha$,
           legend pos=north east,, ymax=0, ymin=0.000000001,
            title=CIFAR-100,  ymode=log]
         \addplot table {lipschitz_estimation/cifar_100/r_0.9.txt};
         \addlegendentry {$d=0.9$}
         \addplot table {lipschitz_estimation/cifar_100/r_0.7.txt};
         \addlegendentry {$d=0.7$}
         \addplot table {lipschitz_estimation/cifar_100/r_0.5.txt};
         \addlegendentry {$d=0.5$}
         \addplot table {lipschitz_estimation/cifar_100/r_0.3.txt};
         \addlegendentry {$d=0.3$}
       \end{axis}
     \end{scope}
     \begin{scope}[scale=1,xshift=11cm]
       \begin{axis}[width=6cm,height=4.5cm, ymax=0, ymin=0.000000001,
           xlabel=$\alpha$, ymode=log,
           legend pos=north east,
            title=Imagenet32]
         \addplot table {lipschitz_estimation/imagenet32/r_0.9.txt};
         \addlegendentry {$d=0.9$}
         \addplot table {lipschitz_estimation/imagenet32/r_0.7.txt};
         \addlegendentry {$d=0.7$}
         \addplot table {lipschitz_estimation/imagenet32/r_0.5.txt};
         \addlegendentry {$d=0.5$}
         \addplot table {lipschitz_estimation/imagenet32/r_0.3.txt};
         \addlegendentry {$d=0.3$}
       \end{axis}
     \end{scope}
    \end{tikzpicture}    
    \vspace{-0.4cm}
    \caption{Depiction of the proportion of pairs of training examples of distinct classes incompatible with the proposed robustness definition for the $\mathcal{L}_{\infty}$ norm, as a function of $\alpha$ and for various values of $d$.}
    \vspace{-0.4cm}
    \label{fig:incompatibility_robustness}
\end{figure*}
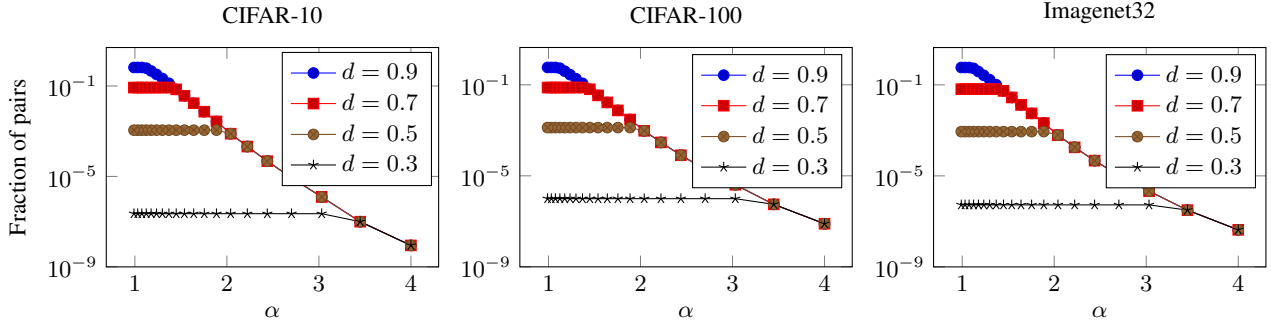

\begin{figure*}[ht]
    \centering
   \begin{tikzpicture}
     \begin{scope}[scale=1]
       \begin{axis}[width=6cm,height=4.5cm,
           xlabel=$SNR$,
           ylabel=Average test accuracy,
           legend pos=south east,
            title=CIFAR-10,
           legend columns=2,
           ]
         \addplot table {snr_accuracy_cifar10/vanilla.txt};
         \addlegendentry {V}
         \addplot table {snr_accuracy_cifar10/parseval.txt};
         \addlegendentry {P}
         \addplot table {snr_accuracy_cifar10/proposed.txt};
         \addlegendentry {L}
         \addplot table {snr_accuracy_cifar10/L2NN.txt};
         \addlegendentry {L2NN}
         \addplot table {snr_accuracy_cifar10/pgd.txt};
         \addlegendentry {PGD}
       \end{axis}
     \end{scope}
     \begin{scope}[scale=1,xshift=5.5cm]
       \begin{axis}[width=6cm,height=4.5cm,
           xlabel=$SNR$,
           legend style={
                at={(0.95,0.65)},
            },
           legend columns=1,
            title=CIFAR-100,
           ]
         \addplot table {snr_accuracy_cifar100/vanilla.txt};
         \addlegendentry {V}
         \addplot table {snr_accuracy_cifar100/parseval.txt};
         \addlegendentry {P}
         \addplot table {snr_accuracy_cifar100/proposed.txt};
         \addlegendentry {L}
       \end{axis}
     \end{scope}
     \begin{scope}[scale=1,xshift=11cm]
       \begin{axis}[width=6cm,height=4.5cm,
           xlabel=$SNR$,
           legend style={
                at={(0.36,0.99)},
            },
           legend columns=1,
            title=Imagenet32,
            xmin=30,
           ]
         \addplot table {snr_accuracy_imagenet32/vanilla.txt};
         \addlegendentry {V}
         \addplot table {snr_accuracy_imagenet32/proposed.txt};
         \addlegendentry {L}
       \end{axis}
     \end{scope}
    \end{tikzpicture}    
    \vspace{-0.4cm}
    \caption{Average test set accuracy under Gaussian noise for various datasets and methods.}
    \vspace{-0.4cm}
    \label{figure_snr}
\end{figure*}
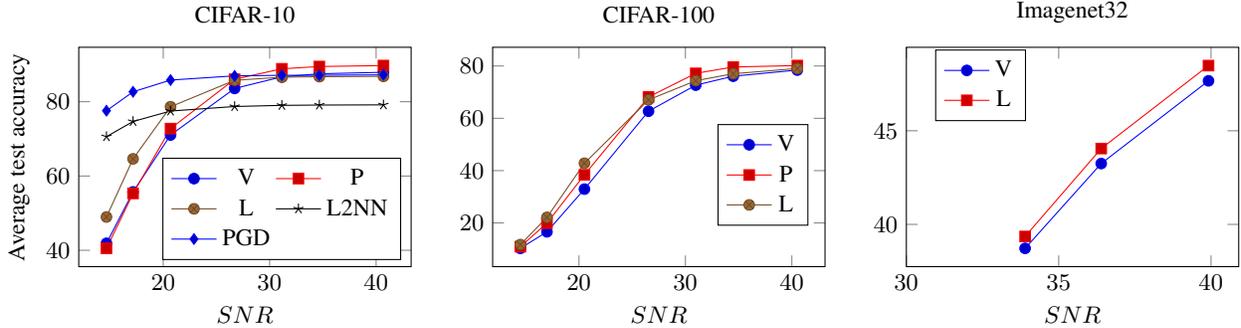

\section{Compositional Robustness}
\label{compositional}

Directly enforcing a robustness criterion on function $F$ can be hard in practice, because there are many stages that map an input vector $\vx$ to the corresponding output $F(\vx)$. Because of this, several works in the literature consider each layer of the architecture separately~\cite{parseval,l2nonexpansive,ours,manifoldmixup}. Following this idea, assume that $F$ is obtained by composing several intermediate functions $$F = f_{\ell_{\max}} \circ \dots \circ f_2 \circ f_1\;,$$ where $\ell_{\max}$ is the depth of the network. For any $\ell$, we denote by $F^\ell = f^\ell \circ \dots \circ f^2 \circ f^1$, so that $F^{\ell_{\max}} = F$. We define 
\emph{layer-robustness} as:
\begin{definition}
We say that $f^\ell$ is $\alpha$-robust over $R$ and for $r>0$ at depth $\ell$ and we denote $f^\ell \in Robust^\ell_\alpha(R,r)$ if:

$$\begin{array}{l}\|f^\ell(\vy+\vepsilon)-f^\ell(\vy)\|\leq \alpha\|\vepsilon\|,\\ \vspace{0.3cm} \forall \vy \in F^{\ell-1}(R), \forall \vepsilon \text{ s.t. } \|\vepsilon\|< r.
\end{array}$$
\end{definition}
\vspace{-.5cm}

There is a direct relation between robustness of functions $f^\ell$ 
at the various layers of the architecture and that of $F$, 
as expressed in the following proposition.

\begin{proposition}
\label{parts_proposition}
Suppose that: $$f^\ell \in \rm{Robust}^\ell_{\alpha^\ell}(R,r \prod_{\lambda\leq \ell-1}{\alpha^{\lambda}}),\forall \ell \text{ s.t. } 1\leq \ell \leq \ell_{max}\;,$$ with $\alpha^\ell \leq1 $ and denote $\alpha = \prod_{\lambda\leq \ell_{\max}}{\alpha^\lambda}$, then $$F = f^{\ell_{\max}} \circ \dots \circ f^2 \circ f^1 \in \rm{Robust}_\alpha(R,r)\;.$$
\end{proposition}

\begin{proof}
The proof is omitted due to lack of space, but available online at \url{https://github.com/cadurosar/structural_robustness/blob/master/Proof.pdf}.
\end{proof}

Note that the condition on $f^\ell$ is less demanding if all previous layers were yielding small values of $\alpha^\ell$ (as the demanded radius for $f^\ell$ robustness is smaller).
We thus observe there would be multiple possible strategies to enforce robustness of $F$ in practice: a) forcing all layers to provide similar robustness or b) focusing only on a few layers of the architecture. Most proposed methods in the literature~\cite{parseval,l2nonexpansive,ours} are aiming at enforcing a).
In fact, option b) could be too restrictive and prevent the learning procedure to converge.

\vspace{-0.2cm}

\section{Relation with existing methods}
\label{regul}
\vspace{-0.2cm}

We evaluate four of the prior works on the literature from the perspective of our proposed robustness measure, namely: Parseval networks (P)~\cite{parseval}, $\mathcal{L}_2$ non-expansive networks (L2NN)~\cite{l2nonexpansive}, Laplacian networks (L)~\cite{ours} and Projected Gradient Descent adversarial data augmentation (PGD)~\cite{madry}. See Table~\ref{table:summary} for a summary.

In~\cite{parseval}, networks are trained to be $\alpha$-Lipschitz for the $\mathcal{L}_2$ and $\mathcal{L}_\infty$ norms in order to achieve robustness. This is achieved by applying a regularizer such that the  weight matrix is, approximately, a Parseval tight frame~\cite{kovavcevic2008introduction}. Among the four methods we consider, this is the only one that leads to improved performance on the clean test set. However, note that \cite{parseval} does not strictly enforce the $\alpha$-Lipschitz constraint, as it disregards batch normalization layers and uses a very small regularization factor. This is why it does not prevent the loss from going to 0 (which theoretically, as seen in  Figure~\ref{fig:incompatibility_lipschitz}, could only be achieved if $\alpha$ is large). This also explains why this method achieves worse results in robustness than L2NN~\cite{l2nonexpansive}. In terms of the proposed definition of robustness, this is a global method that targets the $\text{Robust}_\alpha(r)$ metric for $r \rightarrow +\infty$, penalizing large slopes in the network function between any two points. We will see that more localized approaches (targeting finite $r$) achieve improved robustness. We denote this method P in the remaining of this work.

L2NN~\cite{l2nonexpansive} enforces the network to be $\alpha$-Lipschitz only in terms of the $\mathcal{L}_2$ norm, but does it with a stricter criterion: contrary to P, there is no regularizer to enforce this condition, which is built into the structure of the network itself. \cite{l2nonexpansive} notes that enforcing a global $\alpha$-Lipschitz constant is by itself too hard and that the distances between examples should not collapse throughout the network architecture. As such, they also limit the contraction of space. This seems to be the most robust against $\mathcal{L}_2$ attacks of the four methods we consider. It has also been shown to combine well with PGD training. However, it is also the method that performs the worse on the clean test set.

In~\cite{ours}, we applied a regularization at each ReLU activation in the architecture to enforce that the average distance between examples of different classes remain almost constant from layer to layer. This is achieved by exploiting the smoothness of the class indicator signal across the graph generated by intermediate representations at a given layer. In terms of Definition~\ref{robustness_definition}, this method focuses on pairs of examples of distinct classes 
and tries to restrict changes in their $\mathcal{L}_2$ distance. Thus, \cite{ours} indirectly penalizes changes in local smoothness:  if we consider (\ref{robustness_definition}) with $F(.)$ chosen to be 
the function that assigns to each example its true label, and we do not allow the average $r$ between opposite class examples to change much, then the corresponding $\alpha$ will change slowly with the training.
Note that this approach and Parseval were shown to complement each other in \cite{ours}. We hypothesize that this can be explained because Parseval focuses on global constraints 
 while \cite{ours} favors increased robustness by preserving structure -- and thus smoother network function transitions -- around class boundaries. We denote this method L in the remaining of this work.

Finally, PGD adversarial training~\cite{madry} is a data augmentation procedure that generates adversarial examples during the training phase, by doing multiple iterations of the FGSM method. This leads to a min-max game between the network and the examples generation, and is the best adversarial data augmentation that we are aware of. It works mostly on the domain $T$, as it increases its size and also decreases the difference between $T$ and a noisy test domain. This leads to less domain shift against noisy images on the test, but on the other hand it increases the domain shift to clean images. As a result, the networks perform well against noise (isotropic or adversarial) but have problems with the clean examples. Of the four methods, this is the only to be applied on the harder Imagenet~\cite{kannan2018adversarial} task (but only against the weaker targeted white box attack as noted  in~\cite{engstrom2018evaluating}).

\begin{table}[]
    \centering
    \begin{tabular}{c|c|c|c|c}
\hline
         Method & Domain ($R$) & Slope ($\alpha$) & Locality ($r$) & Metric \\
        \Xhline{2\arrayrulewidth}
         
         P & $\Omega$ & Yes & No & $\mathcal{L}_2$ + $\mathcal{L}_{\infty}$\\
         L2NN & $\Omega$ & Yes & No & $\mathcal{L}_2$\\
         L & $T$ & Approx. & Yes & $\mathcal{L}_2$ + cos\\
         PGD & augmented $T$ & No & Yes & $\mathcal{L}_{\infty}$\\ \hline
    \end{tabular}
    \caption{Summary of the methods and the notions of the introduced robustness they consider.}
    \vspace{-0.5cm}
    \label{table:summary}
\end{table}

\vspace{-0.2cm}

\section{Experiments}
\label{experiments}
\vspace{-0.2cm}

We perform several experiments to evaluate our robustness metric (Definition~\ref{def:robustness}) and its relation to actual network robustness. Vanilla (V), Parseval (P) and Laplacian (L) refer to the networks trained in~\cite{ours}, PGD to the network trained in~\cite{madry} and L2NN to the network trained in~\cite{l2nonexpansive}. Note that this direct comparison with the baseline is not completely fair, as the networks and hyperparameters for different papers are not the same. For example, PGD has more layers and parameters, and uses non-adversarial data augmentation during training, while L2NN does not use a residual architecture. 

Figure~\ref{fig:incompatibility_robustness} shows, as a function of $\alpha$, the ratio between i) the number of examples within distance $d$ of each other that are not $\alpha$-robust (for the $\mathcal{L}_{\infty}$ norm) and ii) the total number of example pairs. Note that $d$ should be roughly interpreted as $2r$ in our definition of robustness.  As in Figure~\ref{fig:incompatibility_lipschitz}, the network function is taken to be the one-hot-bit encoded vectors of the corresponding classes. 
Note that for each choice of $d$ the curve is initially flat and then drops. In the flat section {\em all} pairs within $d$ are {\em not} $\alpha$-robust. Interestingly, this figure shows that the number of pairs of examples in distinct classes that are closer than $d$ in the input space drops very fast as $d$ is decreased. The amount becomes negligible for $d=0.3$ so that it becomes theoretically possible to find a robust network function that is compatible with {\em almost all} pairs of the training set.


Figure~\ref{ralpha} shows the evolution of $\alpha_{\lim}(r)$ as a function of $r$ for the various methods. We use 100 training examples and 1000 Gaussian noise realizations on the CIFAR-10 dataset to estimate $\alpha_{\lim}(\cdot)$. Note that for all methods $\alpha$ increases as a function of $r$ and eventually saturates. Vanilla (V) saturates fastest and at largest value of $\alpha$ because i) sharp transitions in the network function over short distances are allowed and ii) the network function produces outputs closest to the one-hot-bit encoded vector (since V can achieve zero error on the training set). 
In contrast, for all the other methods $\alpha_{\lim}$ grows more slowly with $r$ and  saturates at a lower value, indicating that transitions in the network function are not as sharp and some examples are misclassified. 
Such a compromise between accuracy on the training set and robustness has been discussed in~\cite{Fawzi2018}. 

The fact that, for both P and L, $\alpha$ saturates for larger $r$ suggests that the margin between the examples and the boundary is increased compared to Vanilla. L2NN and PGD saturate at the lowest $\alpha$ values. We observe a transition for PGD occurring at around $r=0.3$ whereas L2NN remains almost constant. This is due to the fact L2NN enforces a strong Lipschitz constraint (using $\mathcal{L}_2$ norm) everywhere on the function: as a result, the network function is almost linear between the training samples. As seen in Figure~\ref{fig:incompatibility_lipschitz}, this creates strong incompatibilities with the training dataset, which is why L2NN achieves the worst performance on the clean set (c.f. Table~\ref{table:robustness}).

We compare methods in terms of robustness on a recently proposed benchmark~\cite{hendrycks2019robustness}. 
The results in Table~\ref{table:robustness} show that PGD achieves the best accuracy and robustness trade-off. Note that for PGD, our robustness metric saturates at a relatively small $\alpha$ and grows for $r$ between 0.2 and 0.4 which correspond to a reasonable range of values in $d\approx2r$ as seen in  Figure~\ref{fig:incompatibility_robustness}.  
Table~\ref{table:robustness} along, with the behavior of PGD in Figure~\ref{ralpha}, suggest that improved robustness is achievable when the network function is smooth locally near the class boundaries, i.e., $\alpha_{\lim}$ grows in typical range separating examples in different classes and saturates at a relative small value. 
Finally, in Figure~\ref{figure_snr} we can see that the relative robustness performance of all methods under Gaussian noise condition is the same as that in Figure~\ref{ralpha}\footnote{Note that we did not report the results for P for the case of Imagenet32 since we did not find right parameters to obtain a good accuracy on the clean test set. Also PGD and L2NN results are not reported in the case of CIFAR-100 and Imagenet32 as pretrained networks were not available.}.

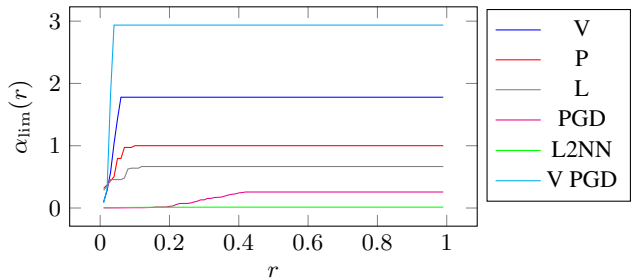
\begin{figure}[ht]
 \begin{center}
   \begin{tikzpicture}[]
       \begin{axis}[width=7cm, height=4.5cm,
           xlabel=$r$,
           ylabel=$\alpha_{\lim}(r)$,
           legend style={at={(1.2,0.99)},anchor=north},
           legend columns = 1]
         \addplot[color=blue] table {r_alpha/vanilla.txt};
         \addlegendentry{V}
         \addplot[color=red] table {r_alpha/parseval.txt};
         \addlegendentry{P}
         \addplot[color=gray] table {r_alpha/laplacian.txt};
         \addlegendentry{L}
         \addplot[color=magenta] table {r_alpha/PGD.txt};
         \addlegendentry{PGD}
         \addplot[color=green] table {r_alpha/L2NN.txt};
         \addlegendentry{L2NN}
         \addplot[color=cyan] table {r_alpha/vanilla_pgd.txt};
         \addlegendentry{V PGD}
       \end{axis}
   \end{tikzpicture}
 \end{center}
    \vspace{-0.3cm}
    \caption{Estimations of $\alpha_{\lim}(r)$ obtained for different radius $r$ over training examples with the $\mathcal{L}_\infty$ norm.}
    \label{ralpha}
\end{figure}


\begin{table}[ht]
\begin{center}
\resizebox{\columnwidth}{!}{%
\begin{tabular}{c|cccccc}
\hline
Dataset          & V       & P                & L       & PGD             & L2NN   & V PGD       \\ 
\Xhline{2\arrayrulewidth}
Clean            & 11.9\%  & \textbf{10.2\%}  & 13.2\%  & 12.8\%          & 20.9\% & 5.0\%        \\
Corrupted        & 31.6\%  & 30.5\%           & 31.3\%  & \textbf{18.8\%} & 28.5\% & 24.4\%     \\
relativeMCE      & 100     & 103              & 92      & \textbf{30}     & 39     & 98     \\  
relativeMCE VPGD & 102     & 105              & 93      & \textbf{31}     & 39     & 100              \\
\hline
\end{tabular}
}
\caption{Test set error on the CIFAR-10 dataset under different image conditions. Corrupted refers to the mean test set error under the 15 image corruption benchmark from~\cite{hendrycks2019robustness}. RelativeMCE for method $X$ is obtained as $100 (Corrupted(X) - Clean(X))/(Corrupted(V) - Clean(V))$.}
\label{table:robustness}
\end{center}
\end{table}


\section{Conclusion}
\label{conclusion}

We have introduced a formal definition of robustness to deviations inside a norm-ball of radius $r$ around the training set. We have shown that this definition can be applied to each part of the network separately. We derived theoretical and empirical relations between our proposed definition and existing methods in the literature.
Future work include looking at other types of perturbations (that do not fall in the norm-ball radius around examples) and fine-tuning already trained networks to improve their robustness.

\bibliography{main}

\begin{thebibliography}{10}

\bibitem{alexnet}
Alex Krizhevsky, Ilya Sutskever, and Geoffrey~E Hinton,
\newblock ``Imagenet classification with deep convolutional neural networks,''
\newblock in {\em Advances in neural information processing systems}, 2012, pp.
  1097--1105.

\bibitem{resnet}
Kaiming He, Xiangyu Zhang, Shaoqing Ren, and Jian Sun,
\newblock ``Identity mappings in deep residual networks,''
\newblock in {\em European Conference on Computer Vision}. Springer, 2016, pp.
  630--645.

\bibitem{shakeshake}
Xavier Gastaldi,
\newblock ``Shake-shake regularization,''
\newblock {\em arXiv preprint arXiv:1705.07485}, 2017.

\bibitem{nmt}
Yonghui Wu, Mike Schuster, Zhifeng Chen, Quoc~V Le, Mohammad Norouzi, Wolfgang
  Macherey, Maxim Krikun, Yuan Cao, Qin Gao, Klaus Macherey, et~al.,
\newblock ``Google's neural machine translation system: Bridging the gap
  between human and machine translation,''
\newblock {\em arXiv preprint arXiv:1609.08144}, 2016.

\bibitem{elmo}
Matthew~E. Peters, Mark Neumann, Mohit Iyyer, Matt Gardner, Christopher Clark,
  Kenton Lee, and Luke Zettlemoyer,
\newblock ``Deep contextualized word representations,''
\newblock in {\em Proc. of NAACL}, 2018.

\bibitem{universalapproximators}
Kurt Hornik, Maxwell Stinchcombe, and Halbert White,
\newblock ``Multilayer feedforward networks are universal approximators,''
\newblock {\em Neural networks}, vol. 2, no. 5, pp. 359--366, 1989.

\bibitem{szegeny}
Christian Szegedy, Wojciech Zaremba, Ilya Sutskever, Joan Bruna, Dumitru Erhan,
  Ian Goodfellow, and Rob Fergus,
\newblock ``Intriguing properties of neural networks,''
\newblock {\em arXiv preprint arXiv:1312.6199}, 2013.

\bibitem{goodfellowadversarial}
Ian~J Goodfellow, Jonathon Shlens, and Christian Szegedy,
\newblock ``Explaining and harnessing adversarial examples,''
\newblock {\em arXiv preprint arXiv:1412.6572}, 2014.

\bibitem{mallat}
St{\'e}phane Mallat,
\newblock ``Understanding deep convolutional networks,''
\newblock {\em Phil. Trans. R. Soc. A}, vol. 374, no. 2065, pp. 20150203, 2016.

\bibitem{hendrycks2019robustness}
Dan Hendrycks and Thomas Dietterich,
\newblock ``Benchmarking neural network robustness to common corruptions and
  perturbations,''
\newblock {\em Proceedings of the International Conference on Learning
  Representations}, 2019.

\bibitem{ford2019adversarial}
Nic Ford, Justin Gilmer, Nicolas Carlini, and Dogus Cubuk,
\newblock ``Adversarial examples are a natural consequence of test error in
  noise,''
\newblock {\em arXiv preprint arXiv:1901.10513}, 2019.

\bibitem{madry}
Aleksander Madry, Aleksandar Makelov, Ludwig Schmidt, Dimitris Tsipras, and
  Adrian Vladu,
\newblock ``Towards deep learning models resistant to adversarial attacks,''
\newblock {\em ICLR}, 2018.

\bibitem{ladder}
Mohammad Pezeshki, Linxi Fan, Philemon Brakel, Aaron Courville, and Yoshua
  Bengio,
\newblock ``Deconstructing the ladder network architecture,''
\newblock in {\em International Conference on Machine Learning}, 2016, pp.
  2368--2376.

\bibitem{kurakin}
Alexey Kurakin, Ian Goodfellow, and Samy Bengio,
\newblock ``Adversarial machine learning at scale,''
\newblock {\em arXiv preprint arXiv:1611.01236}, 2016.

\bibitem{engstrom2018evaluating}
Logan Engstrom, Andrew Ilyas, and Anish Athalye,
\newblock ``Evaluating and understanding the robustness of adversarial logit
  pairing,''
\newblock {\em arXiv preprint arXiv:1807.10272}, 2018.

\bibitem{parseval}
Moustapha Cisse, Piotr Bojanowski, Edouard Grave, Yann Dauphin, and Nicolas
  Usunier,
\newblock ``Parseval networks: Improving robustness to adversarial examples,''
\newblock in {\em International Conference on Machine Learning}, 2017, pp.
  854--863.

\bibitem{l2nonexpansive}
Haifeng Qian and Mark~N. Wegman,
\newblock ``L2-nonexpansive neural networks,''
\newblock in {\em International Conference on Learning Representations}, 2019.

\bibitem{cifar10}
Alex Krizhevsky and Geoffrey Hinton,
\newblock ``Learning multiple layers of features from tiny images,''
\newblock {\em
  https://www.cs.toronto.edu/{\textasciitilde}kriz/learning-features-2009-TR.pdf},
  2009.

\bibitem{chrabaszcz2017downsampled}
Patryk Chrabaszcz, Ilya Loshchilov, and Frank Hutter,
\newblock ``A downsampled variant of imagenet as an alternative to the cifar
  datasets,''
\newblock {\em arXiv preprint arXiv:1707.08819}, 2017.

\bibitem{ours}
Carlos Eduardo Rosar~Kos Lassance, Vincent Gripon, and Antonio Ortega,
\newblock ``Laplacian networks: Bounding indicator function smoothness for
  neural networks robustness,''
\newblock {\em Open Review}, 2019.

\bibitem{manifoldmixup}
Vikas Verma, Alex Lamb, Christopher Beckham, Aaron Courville, Ioannis
  Mitliagkis, and Yoshua Bengio,
\newblock ``Manifold mixup: Encouraging meaningful on-manifold interpolation as
  a regularizer,''
\newblock {\em arXiv preprint arXiv:1806.05236}, 2018.

\bibitem{kovavcevic2008introduction}
Jelena Kova{\v{c}}evi{\'c} and Amina Chebira,
\newblock ``An introduction to frames,''
\newblock {\em Foundations and Trends in Signal Processing}, vol. 2, no. 1, pp.
  1--94, 2008.

\bibitem{kannan2018adversarial}
Harini Kannan, Alexey Kurakin, and Ian Goodfellow,
\newblock ``Adversarial logit pairing,''
\newblock {\em arXiv preprint arXiv:1803.06373}, 2018.

\bibitem{Fawzi2018}
Alhussein Fawzi, Omar Fawzi, and Pascal Frossard,
\newblock ``Analysis of classifiers' robustness to adversarial perturbations,''
\newblock {\em Machine Learning}, vol. 107, no. 3, pp. 481--508, Mar 2018.

\end{thebibliography}
\bibliographystyle{IEEEbib}

\addtolength{\textheight}{-12cm}   

\end{document}